\documentclass[final]{siamonline1116}

\usepackage{braket,amsfonts,amsmath,amssymb}

\usepackage{array}

\usepackage[caption=false]{subfig}

\newsiamthm{claim}{Claim}
\newsiamremark{rem}{Remark}
\newsiamremark{expl}{Example}
\newsiamremark{hypothesis}{Hypothesis}
\newsiamremark{assump}{Assumption}

\crefname{hypothesis}{Hypothesis}{Hypotheses}

\usepackage{algorithmic}

\usepackage{graphicx,epstopdf,color}

\Crefname{ALC@unique}{Line}{Lines}

\usepackage{amsopn}

\numberwithin{theorem}{section}
\def\be{\begin{equation}} \def\ee{\end{equation}}
\newcommand{\proj}{\mathrm{proj}_{\mathcal{Q}}}
\newcommand{\gk}{\gamma_k}
\newcommand{\ak}{\alpha_k}

\newcommand{\dis}{\mathrm{dist}}
\newcommand{\E}{\mathbb{E}}
\newcommand{\R}{\mathbb{R}}

\newcommand{\txk}{\tilde{x}^k}
\newcommand{\0}{\mathbf{0}}
\newcommand{\Q}{\mathcal{Q}}
\newcommand{\Li}{\mathcal{L}}

\newcommand{\la}{\langle}
\newcommand{\ra}{\rangle}

\usepackage{accents}
\newcommand{\ubar}[1]{\underaccent{\bar}{#1}}
\newcommand{\ba}{\bar{\alpha}}
\newcommand{\uba}{\ubar{\alpha}}

\usepackage{hhline}
\usepackage{multirow}
\usepackage{xspace}
\usepackage{bold-extra}
\usepackage[most]{tcolorbox}

\colorlet{texcscolor}{blue!50!black}
\colorlet{texemcolor}{red!70!black}
\colorlet{texpreamble}{red!70!black}
\colorlet{codebackground}{black!25!white!25}

\lstdefinestyle{siamlatex}{%
  style=tcblatex,
  texcsstyle=*\color{texcscolor},
  texcsstyle=[2]\color{texemcolor},
  keywordstyle=[2]\color{texemcolor},
  moretexcs={cref,Cref,maketitle,mathcal,text,headers,email,url},
}

\tcbset{%
  colframe=black!75!white!75,
  coltitle=white,
  colback=codebackground, 
  colbacklower=white, 
  fonttitle=\bfseries,
  arc=0pt,outer arc=0pt,
  top=1pt,bottom=1pt,left=1mm,right=1mm,middle=1mm,boxsep=1mm,
  leftrule=0.3mm,rightrule=0.3mm,toprule=0.3mm,bottomrule=0.3mm,
  listing options={style=siamlatex}
}

\newtcblisting[use counter=example]{example}[2][]{%
  title={Example~\thetcbcounter: #2},#1}

\newtcbinputlisting[use counter=example]{\examplefile}[3][]{%
  title={Example~\thetcbcounter: #2},listing file={#3},#1}

\DeclareTotalTCBox{\code}{ v O{} }
{ 
  fontupper=\ttfamily\color{black},
  nobeforeafter,
  tcbox raise base,
  colback=codebackground,colframe=white,
  top=0pt,bottom=0pt,left=0mm,right=0mm,
  leftrule=0pt,rightrule=0pt,toprule=0mm,bottomrule=0mm,
  boxsep=0.5mm,
  #2}{#1}

\patchcmd\newpage{\vfil}{}{}{}
\begin{tcbverbatimwrite}{tmp_\jobname_header.tex}
\title{BinaryRelax: A Relaxation Approach For Training Deep Neural Networks With Quantized Weights%
  \thanks{P. Yin, S. Zhang, and J. Lyu contributed equally.
\funding{The work was partially supported by NSF grants DMS-1522383, IIS-1632935, ONR grant N00014-16-1-2157, DOE grant DE-SC00183838, and AFOSR grant FA 9550-15-0073.}}}

\author{Penghang Yin%
  \thanks{Department of Mathematics, University of California at Los Angeles, Los Angeles, CA 90095. (\email{yph@ucla.edu, sjo@math.ucla.edu}).}%
  \and
  Shuai Zhang%
  \thanks{Department of Mathematics, University of California at Irvine, Irvine, CA 92697.
    (\email{szhang3@uci.edu, jianchel@uci.edu, yqi@uci.edu, jxin@math.uci.edu}})
  \and
  Jiancheng Lyu%
  \footnotemark[3]
  \and
  Stanley Osher
  \footnotemark[2]
  \and 
  Yingyong Qi
  \footnotemark[3]
  \and
  Jack Xin
  \footnotemark[3]
}

\headers{Binary-Relax}
{Penghang Yin, et al.}
\end{tcbverbatimwrite}
\input{tmp_\jobname_header.tex}

\begin{document}
\maketitle

\begin{tcbverbatimwrite}{tmp_\jobname_abstract.tex}
\begin{abstract}
We propose BinaryRelax, a simple two-phase algorithm, for training deep neural networks with quantized weights. The set constraint that characterizes the quantization of weights is not imposed until the late stage of training, and a sequence of \emph{pseudo} quantized weights is maintained. Specifically, we relax the hard constraint into a continuous regularizer via Moreau envelope, which turns out to be the squared Euclidean distance to the set of quantized weights. The pseudo quantized weights are obtained by linearly interpolating between the float weights and their quantizations. A continuation strategy is adopted to push the weights towards the quantized state by gradually increasing the regularization parameter. In the second phase, exact quantization scheme with a small learning rate is invoked to guarantee fully quantized weights. We test BinaryRelax on the benchmark CIFAR and ImageNet color image datasets to demonstrate the superiority of the relaxed quantization approach and the improved accuracy over the state-of-the-art training methods. Finally, we prove the convergence of BinaryRelax under an approximate orthogonality condition. 
\end{abstract}


\end{tcbverbatimwrite}
\input{tmp_\jobname_abstract.tex}

\section{Introduction}\label{sec:intro}
Deep neural networks (DNNs) have achieved remarkable success in computer vision, speech recognition, and natural language processing systems \cite{imagnet_12,mnist_98,dl_15,faster_rcnn}. There is thus a growing interest in deploying DNNs on low-power embedded systems with limited memory storage and computing power, such as cell phones and other battery-powered devices. However, DNNs typically require hundreds of megabytes of memory storage for the trainable full-precision floating-point parameters or weights, and need billions of FLOPs to make a single inference. This makes the deployment of DNNs impractical on portable devices. Recent efforts have been devoted to the training of DNNs with coarsely quantized weights which are represented using low-precision (8 bits or less) fixed-point arithmetic \cite{bnn_16,bc_15,twn_16,dorefa_16,ttq_16,lbwn_16,inq_17,entropy_17,halfwave_17,yoshida2018ternary,li2017sep}. Quantized neural networks enable substantial memory savings and computation/power efficiency, while achieving competitive performance with that of full-precision DNNs. Moreover, quantized weights can exploit hardware-friendly bitwise operations and lead to dramatic acceleration at inference time.

The simplest way to perform quantization would be directly rounding the weights of a pre-trained full-precision network. But without re-training, this naive approach often leads to poor accuracy at bit-width under 8. From the perspective of optimization, the training of quantized networks can be naturally abstracted as a constrained optimization problem of minimizing some empirical risk subject to a set constraint that characterizes the quantization of weights: 
\begin{equation}\label{eq:training}
\min_{x\in\R^n} \; f(x):=\frac{1}{N}\sum_{j=1}^{N} \ell_j(x) \quad \mbox{subject to} \quad x\in\Q.
\end{equation}
The problem has specific structures. Given a training sample of input $I_j$ and label $u_j$, the corresponding training loss takes the form 
$$
\ell_j(x) = \ell(\sigma_l(x_l*\cdots \sigma_1(x_1*I_j)), u_j),
$$
where $x = [x_{(1)}^\top, \dots, x_{(l)}^\top]^\top$ and $x_{(i)}\in\R^{n_i}$ contains the $n_i$ weights in the $i$-th linear (fully-connected or convolutional) layer with $\sum_{i=1}^l n_i = n$, $\sigma_i$ is some element-wise nonlinear function. "$*$" denotes either matrix-vector product or convolution operation; reshaping is necessary to avoid mismatch in dimensions. For layer-wise quantization, the set $\Q$ takes the form of $\Q_1\times \cdots\times \Q_l$, where $x_{(i)}\in \Q_i:=\mathbb{R_+} \times \{\pm q_1, \pm q_2, \dots, \pm q_m \}^{n_i}$. Here $\mathbb{R}_+$ denotes the set of nonnegative real numbers and $0\leq q_1< q_2< \dots < q_m$ represent the $m$ quantization levels and are pre-determined. The weight vector in the $i$-th layer enjoys the factorization $x_{(i)} = s_i \cdot Q_{(i)}$ for some $Q_{(i)}\in \{\pm q_1, \pm q_2, \dots, \pm q_m\}^{n_i}$ and some trainable layer-wise scalar $s_i\geq 0$. Note that $s_i$ does not have to be low-precision. $s_i$ is shared by all weights across the $i$-th linear layer and will be stored separately from the quantized numbers $Q_i$ for deployment efficiency. The storage for the scaling factors is \emph{negligible} as there are so few of them. Weight quantization has two special cases as follows.
\begin{itemize}
\item 1-bit binarization: $m = 1$ and $\Q_i = \R_{+} \times \{\pm 1\}^{n_i}$. The storage of $Q_{(i)}$'s only needs 1 bit for representing the signs. Compared to the full-precision model, we have $32\times$ memory savings.
\item 2-bit ternarization: $m = 2$ and $\Q_i = \R_{+} \times \{0, \pm 1\}^{n_i}$. The storage needs 2 bits for representing the signs and the binary numbers $\{0,1\}$. Therefore, it gives $16\times$ model compression rate. 
\end{itemize}
The acceleration through low-bit weights is achieved by leveraging the distributive law during forward propagation. For example, propagation through the first linear layer yields the computation of 
$$ 
x_{(1)}*I = (s_1 \cdot Q_{(1)})*I = s_1\cdot(Q_{(1)}*I).
$$
When $Q_{(1)}$ is under 1-bit or 2-bit representation, the computation of $Q_{(1)}*I$ can be extremely fast as there are additions/subtractions involved only.

On the computational side, with sampled mini-batch gradient $\nabla f_k$ at the $k$-th iteration, the classical projected stochastic gradient descent (PSGD) \cite{combettes2015stochastic,rosasco2014convergence}
\begin{equation}\label{eq:psgd}
\begin{cases}
y^{k+1} = x^k - \gk \nabla f_k(x^{k}) \\
x^{k+1} = \proj(y^{k+1}),
\end{cases}
\end{equation}
performs poorly however, and gets stagnated when updated with a small learning rate $\gamma_k$. It is the quantization/projection of weights that ``rounds off'' small gradient updates and causes the plateau as explained by Li et al. in a recent study \cite{Goldstein_17}. Instead of using the standard gradient step in (\ref{eq:psgd}), a hybrid gradient update 
\begin{equation*}
y^{k+1} = y^k - \gk \nabla f_k(x^{k})
\end{equation*}
was adopted by Courbariaux et al. \cite{bc_15} and showed significantly improved accuracy. This modification of PSGD is referred as BinaryConnect in \cite{Goldstein_17}. BinaryConnect has become the workhorse algorithm for training quantized DNN models such as Xnor-Net \cite{xnornet_16} and TWN \cite{twn_16}. By introducing the augmented Lagrangian of (\ref{eq:training}), more complicated algorithms based on alternating minimization were proposed in \cite{carreira1} and \cite{admm_17}. Despite the succinctness and effectiveness of BinaryConnect, its convergence still lacks understanding. The only analysis so far, to our knowledge, appeared in \cite{Goldstein_17} under convexity assumption on the loss function. Researchers have also explored different quantizers, whether uniform or not \cite{bc_15,xnornet_16,dorefa_16,twn_16,lbwn_16,entropy_17,inq_17,carreira2,lacey2018stochastic}. All these methods maintain a sequence of purely quantized weights, if not the optimal, during the training.

In this paper, we propose a novel relaxed quantization approach called BinaryRelax, to explore more freely the non-convex landscape of the objective function of the DNNs under the discrete quantization constraint. We relax the set constraint into a continuous regularizer, which leads to a relaxed quantization update. Besides, we set an increasing regularization parameter, driving $x^k$ slowly to the quantized state. When the training error stops decaying at small $\gamma_k$, we switch to regular quantization to get genuinely quantized weights as desired. By exploiting the structure of quantization set $\Q$, we prove the convergence of BinaryRelax in the non-convex setting, which naturally covers that of BinaryConnect. 

The rest of the paper is organized as follows. In section \ref{sec:br}, we introduce the proposed BinaryRelax method. In section \ref{sec:exper}, we benchmark CIFAR-10 and CIFAR-100 datasets and compare BinaryRelax with state-of-the-art methods to demonstrate the benefits of performing relaxed quantization. In section \ref{sec:converg}, we establish the convergence results. The concluding remarks are given in section \ref{sec:rem}. All technical proofs will be provided in the appendix.

\subsection*{Notations} $\|\cdot\|$ denotes the Euclidean norm; $\|\cdot\|_1$ denotes the $\ell_1$ norm; $\|\cdot\|_0$ counts the number of nonzero components. $\0\in\R^n$ represents the vector of zeros. For any vector $x\in\R^n$ and closed set $\Q\subset \R^n$, 
$$\proj(x):= \arg\min_{z\in\Q}\|x-z\|
$$ 
is the projection of $x$ onto $\Q$, and 
$$
\dis(x,\Q) := \min_{z\in\Q} \|x-z\|
$$
is the Euclidean distance between $x$ and $\Q$. When $\Q$ is a subspace in $\R^n$, $x\perp \Q$ means that $x$ is orthogonal to $\Q$. $\mathrm{sign}(x)$ is the signum function acting pointwise on $x$, i.e.,
\begin{equation*}
\mathrm{sign}(x)_i :=
\begin{cases}
1  & \mbox{if} \; x_i >0, \\
-1 & \mbox{if} \; x_i <0, \\
0 & \mbox{if} \; x_i = 0.
\end{cases}
\end{equation*}

\bigskip

\section{BinaryRelax}\label{sec:br}
Without loss of generality, we assume the set of quantized weights 
$$
\Q= \R_+ \times \{\pm q_1, \dots, \pm q_m \}^{n}
$$ 
throughout the paper. With that said, we only consider the case for simplicity that a single adjustable scaling factor is shared by all weights in the network.

\subsection{Quantization}\label{subsec:quant} 
In fact, for general $b$-bit quantization, 
$$\Q = \bigcup_{i=1}^p\Li_i$$ 
is the union of $p$ distinct one-dimensional subspaces $\Li_i\subset\R^n$, $i = 1, 2, \dots, p$, where 
$$
\Li_i = \{s \cdot L_i: s\in\R\}
$$ 
for some $L_i\in \{\pm q_1, \dots, \pm q_m \}^{n}\setminus\{\0\} \subset \R^n$. See also \cite{admm_17}. Fig. \ref{fig:quant} shows an example of $\Q= \mathbb{R}_+ \times \{0, \pm 1\}^2$, i.e., the ternarization of two weights.

\begin{figure}
\begin{center}
\includegraphics[width=.53\textwidth,height = 0.45\textwidth]{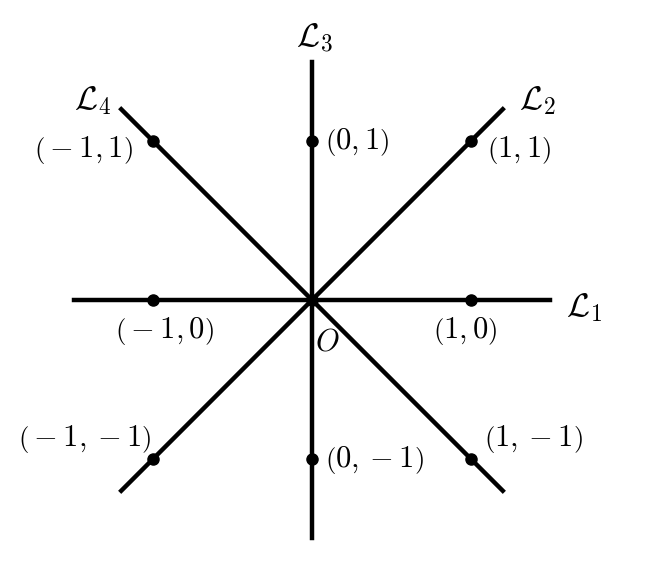}
\end{center}
\caption{Graphic illustration of $\Q= \mathbb{R}_+ \times \{0, \pm 1\}^2$. In this case, $b= n =2$, $p=4$.}\label{fig:quant}
\end{figure}

Given a float weight vector $y$, its quantization $x$ is basically the projection of $y$ onto the set $\Q$, which gives rise to the optimization problem
\begin{equation}\label{eq:proj}
x = \arg\min_{z\in\Q} \; \|z- y\|^2 = \proj(y).
\end{equation}
Note that $\Q$ is a non-convex set, then the projection may not be unique. In that case, we just assume $x$ is one of them. The above projection/quantization problem can be reformulated as 
\begin{equation}\label{eq:quant}
(s^*, \, Q^*) = \arg\min_{s, Q} \; \|s\cdot Q -y\|^2 \quad \mbox{subject to} \quad Q\in \{\pm q_1, \dots, \pm q_m \}^{n},
\end{equation}
The quantization of $y$ is then given by $\proj(y) = s^*\cdot Q^*$. (\ref{eq:quant}) is essentially a constrained $K$-means clustering problem of 1-D points.  The centroids are of the form $\pm (s \cdot q_j)$ with $1\leq j\leq m$, and they are determined by a single parameter $s$ since $q_j$'s are fixed. For uniform quantization where $q_j = j-1$, these centroids are equi-spaced. Given $s$, the assignment of float weights is then governed by $Q$. So the problem (\ref{eq:quant}) in principle can be solved by a variant of Lloyd's algorithm \cite{lloyd}, which iterates between the assignment step ($Q$-update) and centroid update step ($s$-update). In the $Q$-update of the $l$-th iteration, fixing the scaling factor $s^{l-1}$, each $Q_i^{l}$ is chosen from $\{\pm q_1, \dots, \pm q_m \}$ so that $s^{l-1} Q_i^{l}$ is the nearest centroid to $y_i$. In the $s$-update, a quadratic problem 
$$
\min_{s\in\R} \; \|s\cdot Q^{l} -y\|^2  
$$
is solved by $s^{l} = \frac{\langle Q^{l}, y \rangle}{\|Q^{l}\|^2}$.

The above procedure however, is impractical here, as the quantization is needed in every iteration of training. It has been shown that the closed form (exact) solution of (\ref{eq:quant}) can be computed at $O(n)$ complexity for binarization \cite{xnornet_16} where $Q\in\{\pm 1\}^n$:
\begin{equation}\label{eq:bin}
s^* = \frac{\|y\|_1}{n}, \; Q_i^* = 
\begin{cases}
1 & \mbox{if} \; y_i \geq 0\\
-1 & \mbox{otherwise.}
\end{cases}
\end{equation}
In the case of ternarization where $Q\in\{0, \pm 1\}^n$, an $O(n\log n)$ exact formula was found in \cite{lbwn_16}:
\begin{equation}\label{eq:ter}
t^* = \arg\max_{1\leq t\leq n} \frac{\|y_{[t]}\|_{1}^2}{t}, \; s^* = \frac{\|y_{[t^*]}\|_{1}}{t^*}, \; Q^* = \mathrm{sign}(y_{[t^*]}),
\end{equation}
where $y_{[t]}\in\R^n$ keeps the $t$ largest component in magnitude of $y$, while zeroing out the others. For quantization with wider bit-width ($b>2$), accurately solving (\ref{eq:quant}) becomes computationally intractable \cite{lbwn_16}. Empirical formulas have thus been proposed for an approximate quantized solution \cite{twn_16,lbwn_16,inq_17}, and they turn out to be sufficient for practical use. 
For example, a thresholding scheme of $O(n)$ complexity for ternarization was proposed in \cite{twn_16} as
\begin{equation}\label{eq:approx_ter}
\delta = \frac{0.7\|y\|_1}{n}, \; s^* = \frac{\sum_{i=1}^n |y_i|\cdot 1_{|y_i| \geq \delta} }{\sum_{i=1}^n 1_{|y_i| \geq \delta}}, \;
Q^*_i = \begin{cases}
\mathrm{sign}(y_i) & \mbox{if} \; |y_i| \geq \delta \\
0 & \mbox{otherwise.}
\end{cases}
\end{equation}
For $b>2$, Yin et al. \cite{lbwn_16} proposed to just perform one iteration of Lloyd's algorithm with a carefully initialized $Q$.

The focus of this paper is not on how to quantize a float weight vector. From now on, we simply assume that the quantization $\proj(y)$ can be computed precisely, regardless the choice of $q_j$'s.

\subsection{Moreau envelope and proximal mapping}
In the seminal paper \cite{Moreau_65}, Moreau introduced what is now called the Moreau envelope and the proximity operator (a.k.a. proximal mapping) that generalizes the projection. Let $g: \R^n \to (-\infty, \infty]$ be a lower semi-continuous extended-real-valued function. For any $t >0$, the Moreau envelope function $g_t$ is defined by 
$$
g_t(x) :=\inf_{z\in\R^n} \; g(z) + \frac{1}{2t}\|z-x\|^2.
$$
In general, $g_t$ is everywhere finite and locally Lipschitz continuous. Moreover, $g_t$ converges pointwise to $g$ as $t\to 0^+$. Moreau envelope is closely related to the inviscid Hamilton-Jacobi equation \cite{deep_relax}
$$
u_t + \frac{1}{2}\, |\, \nabla_x u\,|^2 = 0, \;\; u(x,0) = g(x),
$$
where $u(x,t) = g_t(x)$ is the unique viscosity solution of the above initial-value problem via the Hopf-Lax formula
$$
u(x, t) = \inf_z \; \left\{ g(z) + t H^*\left(\frac{z-x}{t}\right)  \right\}
$$
with the Hamiltonian $H(t,x,v) = \frac{1}{2}\|v\|^2$ and its Fenchel conjugate $H^* = H$. The proximal mapping of $g$ is defined by
$$
\mathrm{prox}_g(x) := \arg\min_{z\in\R^n} \; g(z) + \frac{1}{2}\|z-x\|^2.
$$
It is frequently used in optimization algorithms associated with non-smooth optimization problems such as total variation denoising \cite{GO_09}.

In particular, if $g=\chi_\mathcal{A}$ is the indicator function of a close set $\mathcal{A}\subset\R^n$, where 
\begin{equation*}
\chi_\mathcal{A}(x) = 
\begin{cases}
0 & \mbox{if } x\in \mathcal{A}\\
\infty & \mbox{otherwise.}
\end{cases}
\end{equation*}
The Moreau envelope is well defined for $t>0$ and is given by
$$
\inf_z \; \chi_\mathcal{A}(z) + \frac{1}{2t}\, \|z-x\|^2 = \inf_{z\in\mathcal{A}} \; \frac{1}{2t}\, \|z-x\|^2 = \frac{1}{2t}\dis(x,\mathcal{A})^2.
$$
And the proximal mapping $\mathrm{prox}_g(x)$ reduces to the projection $\mathrm{proj}_{\mathcal{A}}(x)$.

\subsection{Relaxed quantization} 
Let us begin with the alternative form of DNNs quantization problem (\ref{eq:training}):
\begin{equation}\label{eq:quant2}
\min_{x\in\R^n} \; f(x) + \chi_\Q(x),
\end{equation}
When both the objective function $f(x)$ and the set $\Q$ are non-convex, the discontinuity of $\chi_\Q$ poses an extra challenge in minimization since a continuous gradient descent update can be made stagnant when projected discontinuously. The Moreau envelope of $\chi_\Q$ is $\frac{1}{2t}\dis(x,\Q)^2$, which is continuously differentiable almost everywhere, except at points that have at least two nearest line subspaces, i.e., there exist two different ways to quantize $x$. We use $\frac{1}{2t}\dis(x,\Q)^2$ as the approximant of the discontinuous $\chi_\Q(z)$ and propose to minimize the relaxed training error
\begin{equation}\label{eq:binaryrelax}
\min_{x\in\R^n} \; f(x) + \frac{\lambda}{2} \dis(x,\Q)^2,
\end{equation}
where $\lambda = t^{-1}>0$ is the regularization parameter. 
When $\lambda\to\infty$, $\frac{\lambda}{2} \dis(x,\Q)^2$ converges pointwise to $\chi_\Q(x)$, and the global minimum of (\ref{eq:binaryrelax}) converges to that of (\ref{eq:quant2}).
\begin{proposition}\label{prop:global}
Suppose $f(x)$ is continuous. Let $f_\Q^* = \min_{x\in\Q} f(x)$ be the global minimum of (\ref{eq:quant2}) and $x^*_\lambda$ be the global minimizer of relaxed quantization problem (\ref{eq:binaryrelax}). Then
$$
\dis(x^*_\lambda, \Q) \to 0 \mbox{ and } f(x_\lambda^*)\to f_\Q^*, \; \mbox{as } \lambda\to\infty.
$$
\end{proposition}

\begin{rem}
Relaxation via Moreau envelope leads to a quadratic penalty formulation. An interesting but different quadratic penalty was considered in \cite{carreira1} for the general compression problem of DNNs. In fact, by replacing the Euclidean distance $\|\cdot\|$ in Moreau envelope
with the metric $\|\cdot\|_D$ induced by some matrix $D$, one can derive a more general penalty
$$
\inf_{z\in\mathcal{A}} \; \frac{1}{2}\|z-x\|_D^2.
$$
We refer the readers to the recent paper \cite{GPS} for successful application of such penalty to phase retrieval problem.
\end{rem}

\subsection{Algorithm}
Inspired by the hybrid gradient update proposed in \cite{bc_15}, we write a two-line solver for the minimization problem (\ref{eq:binaryrelax}):
\begin{equation}\label{eq:relax_proj}
\begin{cases}
y^{k+1} = y^k - \gk \nabla f_k(x^{k}) \\
x^{k+1} = \arg\min_{x\in\R^n} \frac{1}{2}\, \|x-y^{k+1}\|^2 + \frac{\lambda}{2} \dis(x,\Q)^2.
\end{cases}
\end{equation}
The algorithm constructs two sequences: an auxiliary sequence of float weights $\{y^k\}$ and a sequence of \emph{nearly} quantized weights $\{x^k\}$. The mismatch of discontinuous projection and continuous gradient descent is resolved by the relaxed quantization step in (\ref{eq:relax_proj}),
which calls for computing the proximal mapping of the function $\frac{\lambda}{2}\dis(x,\Q)^2$. This can be done via the following closed-form formula.
\begin{proposition}\label{prop:approx_quant}
Let 
$$
\proj(y^{k+1}) = \arg\min_{x\in\Q} \|x-y^{k+1}\|^2
$$ 
be the quantization of $y^{k+1}$, then the solution to relaxed quantization subproblem in (\ref{eq:relax_proj}) is given by
\begin{equation}\label{eq:relax_quant}
x^{k+1} = \frac{\lambda\, \proj(y^{k+1}) + y^{k+1}}{\lambda+1}.
\end{equation}
\end{proposition}
Note that we still need the exact quantization $\proj(y^{k+1})$ to perform relaxed quantization. The update $x^{k+1}$ is essentially a linear interpolation between $y^{k+1}$ and its quantization $\proj(y^{k+1})$, and $\lambda$ controls the weighted average. $x^{k+1}$ is thus not quantized because $x^{k+1}\not\in \Q$, but $x^{k+1}$ approaches $\Q$ as $\lambda$ increases. Hereby we adopt a continuation strategy and let $\lambda $ grow slowly. Specifically, we inflate $\lambda$ after a certain number of epochs by a factor $\rho > 1$. Intuitively, the relaxation with continuation will help skip over some bad local minima of (\ref{eq:quant2}) located in $\Q$ , because they are not local minima of the relaxed formulation in general.
\begin{proposition}\label{thm:local}
Suppose $f(x)$ is differentiable. Any point $x^*\in\Q$ is not a local minimizer of the relaxed quantization problem (\ref{eq:binaryrelax}) unless $\nabla f(x^*)= \0$.
\end{proposition}

In order to obtain quantized weights in the end, we turn off the relaxation mode and enforce quantization. The BinaryRelax algorithm is summarized in Alg.~\ref{alg}.

\begin{algorithm}
\caption{BinaryRelax.}\label{alg}
\textbf{Input}: number of epochs for training, batch size, schedule of learning rate $\{\gamma_k\}$, growth factor $\rho>1$.
\begin{algorithmic}
    \FOR {i = 1, 2,$\dots$, nb-epoch} 
    \STATE Randomly shuffle the data and partition into batches.
    \FOR {j = 1, 2, $\dots$, nb-batch}
    \STATE $y^{k+1} = y^{k} - \gk \nabla f_k(x^k) $
    \IF {$i \leq T$}
    \STATE $x^{k+1} = \frac{\lambda_{k}\proj(y^{k+1}) + y^{k+1}}{\lambda_{k}+1}$ \quad $//$ Phase I
    \IF {increase $\lambda$}
    \STATE $\lambda_{k+1} = \rho\lambda_k$
    \ELSE
    \STATE $\lambda_{k+1} = \lambda_k$
    \ENDIF
    \ELSE
    \STATE $x^{k+1} = \proj(y^{k+1})$ \quad $//$ Phase II
    \ENDIF
    \STATE $k = k+1$
    \ENDFOR
    \ENDFOR
\end{algorithmic}
\end{algorithm}

\begin{rem}
For BinaryRelax, we replace a discrete quantization constraint with a continuous regularizer. The similar idea of relaxing the discrete sparsity constraint $\|x\|_0\leq s$ into a continuous and possibly non-convex sparse regularizer has been long known in the contexts of statistics and compressed sensing \cite{lasso,scad,cs}. For example, compressed sensing solvers for minimizing the convex $\ell_1$ norm \cite{GO_09} or non-convex sparse proxies, such as $\ell_{1/2}$ (with smoothing) \cite{irls} and $\ell_{1-2}$ \cite{l1_2}, often empirically outperform those directly tackling the nonzero counting metric $\ell_0$. Interestingly, similar to the quantization set $\Q$, the sparsity constraint set $\{x\in\R^n: \|x\|_0\leq s  \}$ is also a finite union of low-dimensional subspaces in $\R^n$. More precisely, 
$$
\{x\in\R^n: \|x\|_0\leq s  \} = \bigcup_{\mathcal{S}\subset\{1,\dots,n\}, \, |\mathcal{S}|=s}\{x\in\R^n: \mathrm{supp}(x)\subseteq \mathcal{S}\},
$$
where $\mathrm{supp}(x)$ denotes the support of $x$, and each member in the union is a $s$-dimensional subspace with $s\ll n$.
\end{rem}

\begin{rem}
BinaryRelax resembles the linearized Bregman algorithm proposed by Yin et al. \cite{yin2010analysis,linBreg} for solving the basis pursuit problem \cite{chen2001atomic,cs}
\begin{equation*}
\min_{x\in\R^n} \; \|x\|_1 \quad \mbox{subject to} \quad Ax = b,
\end{equation*}
by iterating
\begin{equation*}
\begin{cases}
y^{k+1} =y^k - \tau_k A^\top(A x^k - b) \\
x^{k+1} =\delta \cdot \mathrm{shrink}(y^{k+1},\mu)
\end{cases}
\end{equation*}
where $\delta, \, \mu, \, \tau_k>0$ are algorithmic parameters. In linearized Bregman, $A^\top(A x - b)$ is the gradient of sum of squares error $\frac{1}{2}\|A x - b\|^2$, and $\mathrm{shrink}(y,\mu)$ is the proximal mapping of $\ell_1$ norm (a.k.a. soft-thresholding operator \cite{shrink}):
$$
\mathrm{shrink}(y,\mu):= \arg\min_u \; \frac{1}{2\mu}\|u-y\|^2 + \|u\|_1.
$$
With that said, linearized Bregman also iterates between some sort of hybrid gradient step and proximal mapping. However, it is not exactly the same as BinaryRelax, as there is a scaling by $\delta$ in the proximal step. 
\end{rem}

\subsection{Connection to BinaryConnect}
In fact, the Phase II of BinaryRelax 
\begin{equation}\label{eq:binaryconnect}
\begin{cases}
y^{k+1} = y^k - \gk \nabla f_k(x^{k}) \\
x^{k+1} = \proj(y^{k+1})
\end{cases}
\end{equation}
is exactly the BinaryConnect scheme \cite{bc_15}. The performance of BinaryRelax, however, mostly relies on Phase I training. As will be seen from the experimental results reported in section \ref{sec:exper}, the gain from Phase II training is very limited. Switching to BinaryConnect in Phase II is just to get truly quantized weights.

\bigskip


\section{Experimental Results}\label{sec:exper}
We tested BinaryRelax on benchmark CIFAR \cite{cifar_09} and ImageNet \cite{imagenet_09} color image datasets. The two baselines are the BinaryConnect framework  (\ref{eq:binaryconnect}) combined with the exact binarization formula (\ref{eq:bin}) (BWN) \cite{xnornet_16} and the heuristic ternarization scheme (\ref{eq:approx_ter}) (TWN) \cite{twn_16}, resp.. We used the same quantization formulas for BinaryRelax in the relaxed quantization update (\ref{eq:relax_quant}). Both algorithms were initialized with the weights of a pre-trained float model.

\subsection{The selection of \texorpdfstring{$\lambda$}{Lg}} We always initialize the relaxation parameter $\lambda_0 = 1$. We split into roughly 4/5 and 1/5 of the training epochs for Phase I and Phase II, resp.. To guarantee the smooth transitioning to Phase II from Phase I, a proper growth factor $\rho>1$ is chosen so that $\lambda\in(100,200)$ at the moment Phase I ends. A relatively small $\lambda$ will result in noticeable drop in accuracy when Phase II starts.


\subsection{CIFAR datasets}
The CIFAR-10 dataset consists of 60000 32$\times$32 color images in 10 classes, with 6,000 images per class. There are 50,000 training images and 10,000 test images. CIFAR-100 dataset is just like the CIFAR-10, except it has 100 classes containing 600 images each. There are 500 training images and 100 test images per class. Fig.~\ref{fig:img} shows some sample images from CIFAR datasets. In the experiments, we used the testing images for validation. We coded up the BinaryRelax in PyTorch \cite{pytorch} platform. The experiments were carried out on two desktops with Nvidia graphics cards GTX 1080 Ti and Titan X, resp..

\begin{figure}
\begin{center}
\begin{tabular}{cc}
\textbf{CIFAR-10} & \textbf{CIFAR-100} \\
\includegraphics[scale=0.45]{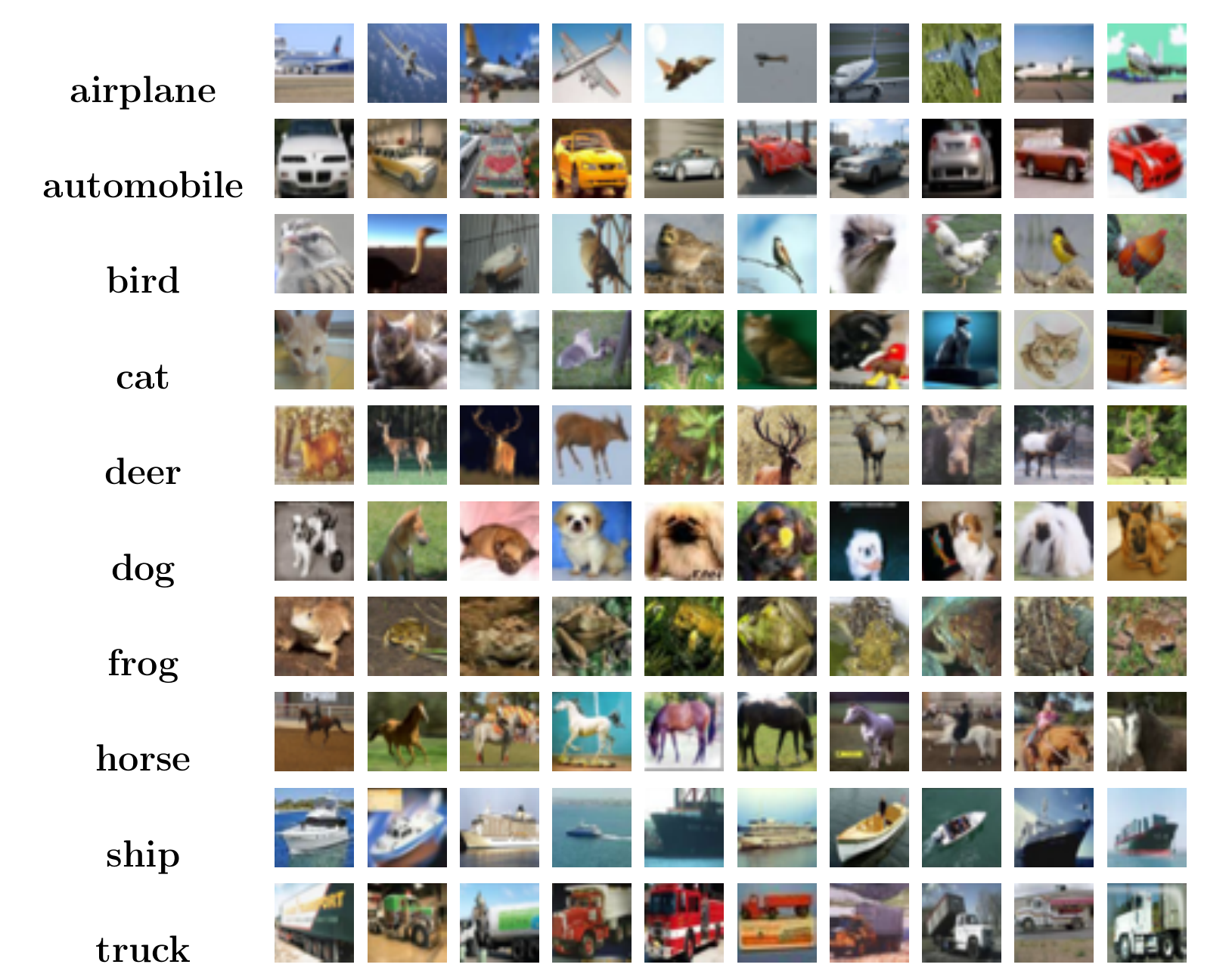} & 
\includegraphics[scale=0.45]{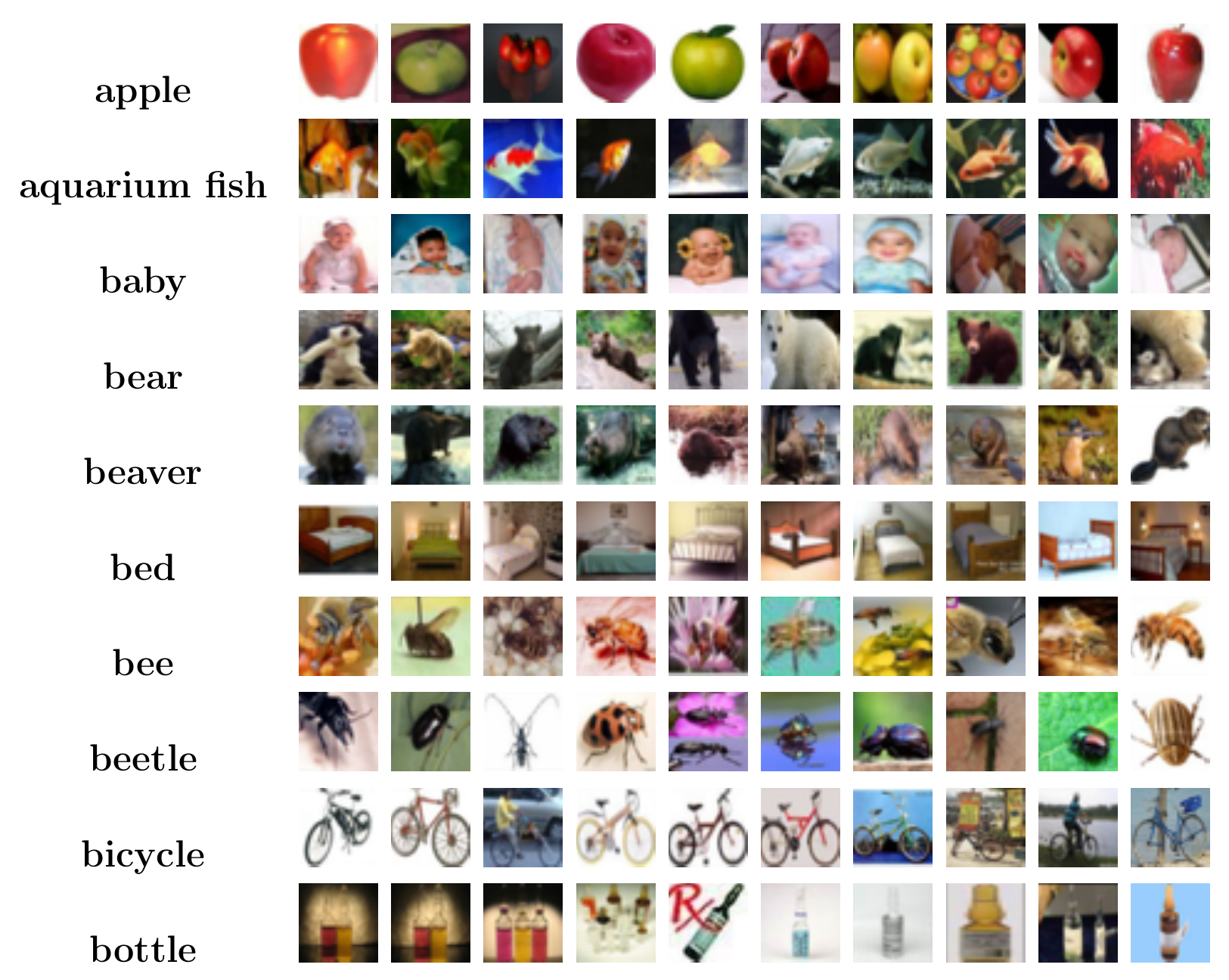}
\end{tabular}
\end{center}
\caption{Sample images from CIFAR datasets: 10 classes in CIFAR-10 (left); 10 out of 100 classes in CIFAR-100 (right).}\label{fig:img}
\end{figure}

We ran 300 epochs. The initial learning rate $\gamma_0= 0.1$ with decay by a factor of 0.1 at epochs $\{120, 220\}$. Phase II starts at epoch 240. $\lambda$ increases by a factor of $\rho = 1.02$ after every epoch. In addition, we used batch size $= 128$, $\ell_2$ weight decay $=10^{-4}$, batch normalization \cite{bn_15}, and momentum $= 0.95$. 

We tested the algorithms on the popular VGG \cite{vgg_14} and ResNet\cite{resnet_15} architectures, and the validation accuracies for CIFAR-10 and CIFAR-100 are summarized in Tab. \ref{tab:cifar10} and Tab. \ref{tab:cifar100}, resp.. Note that ResNet-18 and ResNet-34 tested here were originally constructed for the more challenging ImageNet classification \cite{imagenet_09} and then adapted for CIFAR datasets. They have wider channels in the convolutional layers and are much larger than the other ResNets. For example, ResNet-18 has $\sim 11$ million parameters, whereas ResNet-110 has only $\sim 1.7$ million. This explains their higher accuracies. All quantized networks were initialized from their full-precision counterparts whose validation accuracies are listed in the second column. Fig. \ref{fig:curve} shows the validation accuracies for CIFAR-100 tests with VGG-16 and ResNet-34 during the training process. For the VGG-16 tests, we notice the decay of the validation accuracies of BinaryRelax occurs in Phase I training. This is due to the increase of the parameter $\lambda$, which makes the regularization more and more stringent. With approximately the same training cost, our relaxed quantization approach consistently outperforms the hard quantization counterpart in validation accuracies. As seen from the tables and figure, the advantage of relaxed quantization is particularly clear when it comes to the large nets ResNet-18 and ResNet-34, where we have more complex landscapes with spurious local minima. In this case, our accuracies of binarized networks even surpass that of TWN. The relaxation indeed helps skip over bad local minima during the training.

\begin{table*}[ht]
\label{tab:cifar10}
\centering
\begin{tabular}{|c|c|c|c|c|c|c|c|}
  \hline			
 \multirow{2}{*}{\textbf{CIFAR-10}} & \multirow{2}{*}{Float} & \multicolumn{2}{c|}{Binary} & \multicolumn{2}{c|}{Ternary}  \\
  \cline{3-6}
  & & BWN & Ours & TWN & Ours \\
  \hline
  VGG-11 & 91.93 & 88.70 & {\bf 89.28} & 90.48& {\bf 91.01} \\
  \hline
  VGG-16 & 93.59 & 91.60 & {\bf 91.98} & 92.75 & {\bf 93.20}\\
  \hline
  ResNet-20 & 92.68 & 87.44  & {\bf 87.82} & 88.65 & {\bf 90.07} \\
  \hline
  ResNet-32 & 93.40  & 89.49 & {\bf 90.65} & 90.94 & {\bf 92.04} \\
  \hline
  ResNet-18 & 95.49 & 92.72 & {\bf 94.19}  & 93.55 & {\bf 94.98} \\
  \hline
  ResNet-34 & 95.70 & 93.25 & {\bf 94.66} & 94.05 & {\bf 95.07}  \\
  \hline
\end{tabular}
\medskip
\caption{CIFAR-10 validation accuracies.}
\end{table*}

\begin{table*}[ht]
\label{tab:cifar100}
\centering
\begin{tabular}{|c|c|c|c|c|c|c|c|}
  \hline			
 \multirow{2}{*}{\textbf{CIFAR-100}} & \multirow{2}{*}{Float} & \multicolumn{2}{c|}{Binary} & \multicolumn{2}{c|}{Ternary}  \\
  \cline{3-6}
  & & BWN & Ours & TWN & Ours \\
  \hline
  VGG-11 & 70.43& 62.35 & {\bf 63.82} & 64.16 & {\bf 65.87}\\
  \hline
  VGG-16 & 73.55 & 69.03  & {\bf 70.14} &  71.41 & {\bf 72.10}\\
  \hline
  ResNet-56 & 70.86 & 66.73  & {\bf 67.65}  &  68.26 & {\bf 69.83} \\
  \hline
  ResNet-110 & 73.21 & 68.67  & {\bf 69.85} & 68.95  & {\bf 72.32} \\
  \hline
  ResNet-18 & 76.32 & 72.31 & {\bf 74.04}  & 73.15 & {\bf 75.24} \\
  \hline
  ResNet-34 & 77.23 & 72.92 & {\bf 75.62}  & 74.43 & {\bf 76.16} \\
  \hline
\end{tabular}
\medskip
\caption{CIFAR-100 validation accuracies.}
\end{table*}

\begin{figure}[ht]
\centering
\begin{tabular}{cc}
\textbf{VGG-16 Binary} & \textbf{VGG-16 Ternary} \\
\includegraphics[scale=0.5]{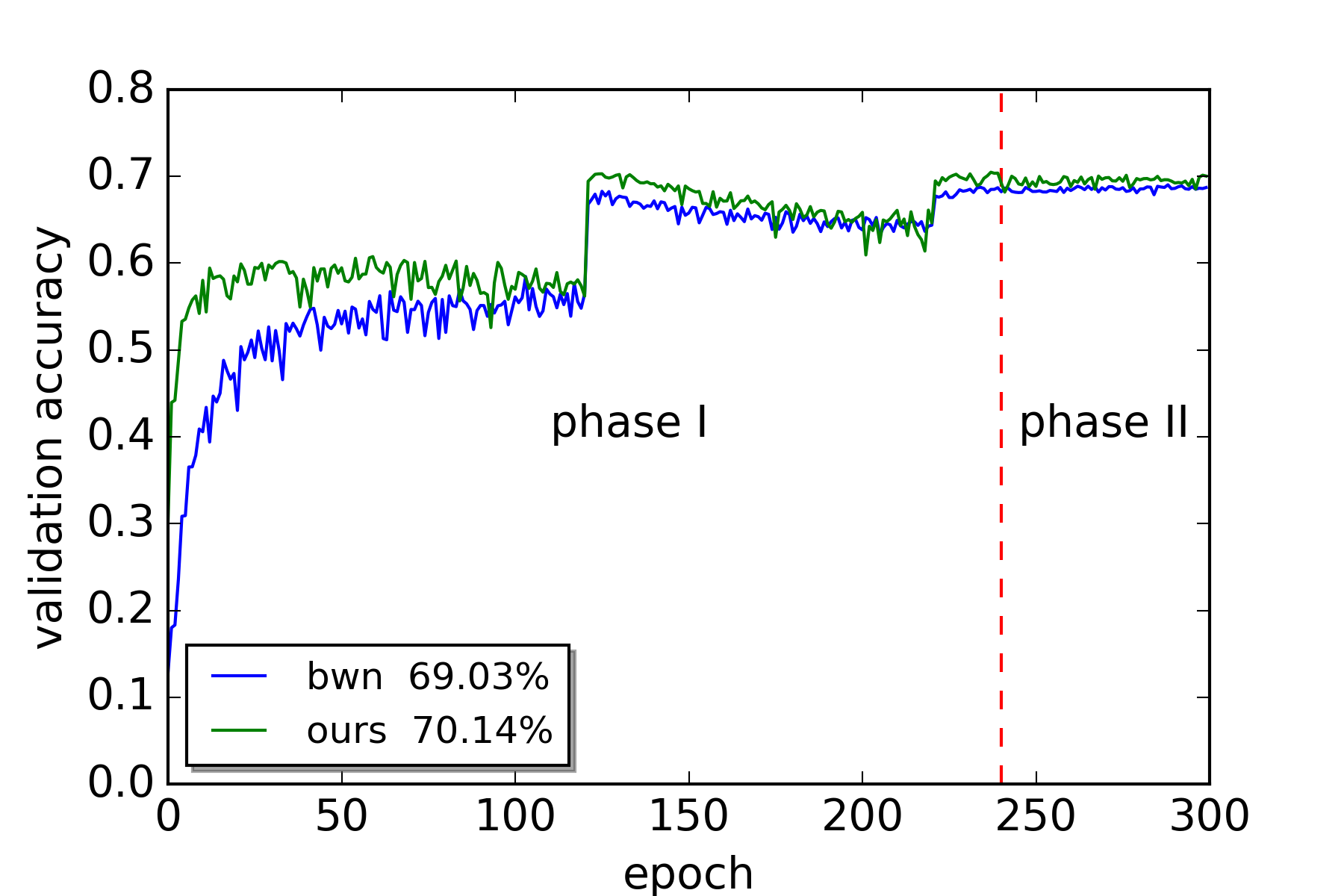} &
\includegraphics[scale=0.5]{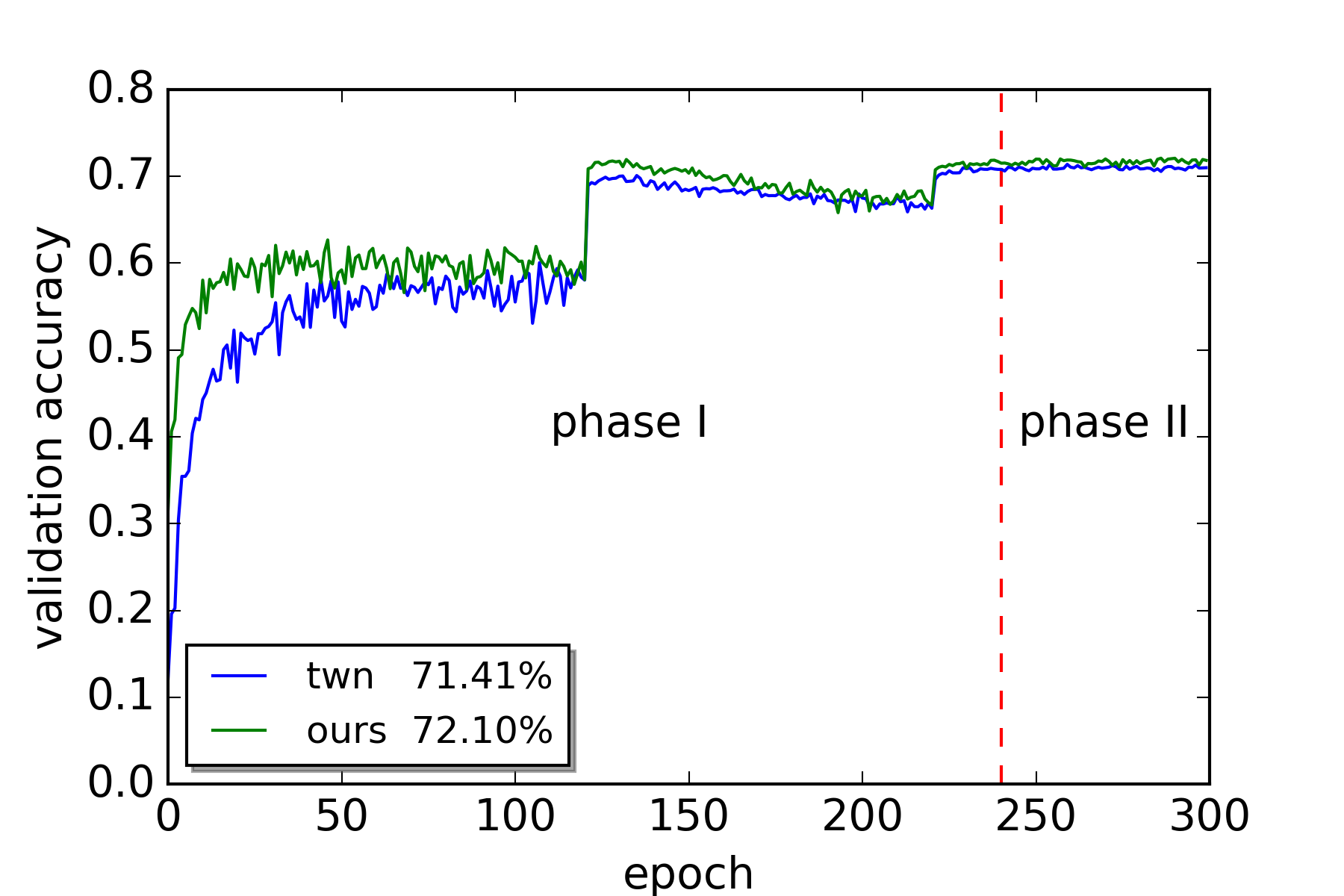}  \\
\textbf{ResNet-34 Binary} & \textbf{ResNet-34 Ternary} \\
\includegraphics[scale=0.5]{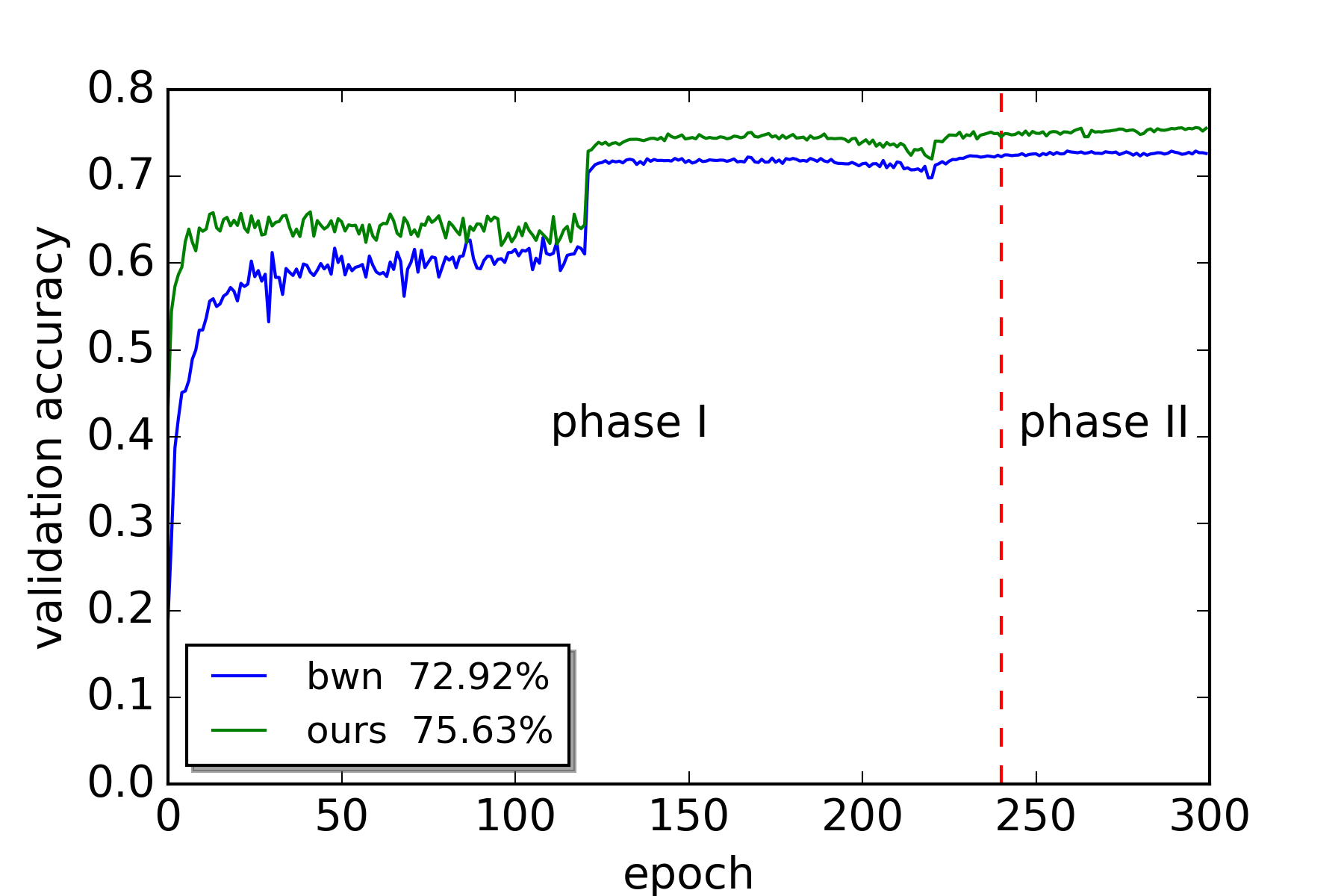} &
\includegraphics[scale=0.5]{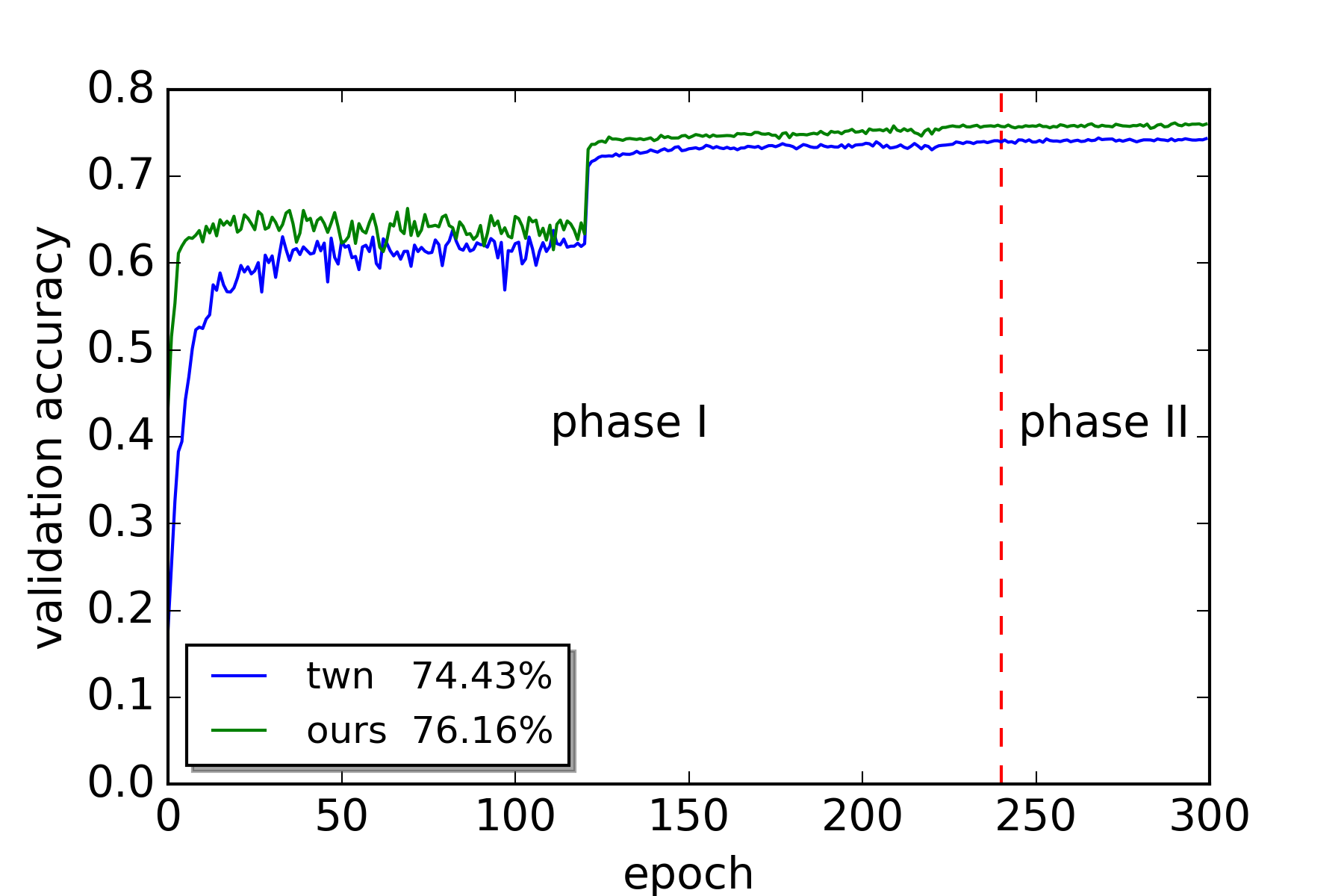}  \\
\end{tabular}
\caption{Comparisons of validation accuracy curves for CIFAR-100 using VGG-16 and ResNet-34. The initial learning rate $\gamma_0=0.1$ and decays by a factor of $0.1$ at epoch $120$ and $220$. The initial regularization parameter $\lambda_0 = 1$ and grows by a factor of $\rho = 1.02$ after each epoch until epoch $240$ where Phase II starts.}
\label{fig:curve}
\end{figure}

\subsection{ImageNet}
ImageNet (ILSVRC12) dataset \cite{imagenet_09} is a benchmark for large-scale image classification task, which has $1.2$ million images for training and $50,000$ for validation of 1,000 categories. We quantize ResNet-18 at bit-widths 1 (binary) and 2 (ternary). The experiments were carried out on a machine with 8 Nvidia GeForce GTX 1080 Ti GPUs.

We initialized BinaryRelax with the pre-trained full-precision (32-bit) models available from the PyTorch torchvision package \cite{pytorch}. We trained in total $70$ epochs, with phase II starting at epoch 55. The initial learning rate $\gamma_0 = 0.1$ and decays by a factor of $0.1$ at epochs $\{30, 40, 50\}$. Relaxation parameter $\lambda$ starts at 1 and increases by a growth factor of $\rho = 1.045$ after each half (1/2) epoch. In all these experiments, we used momentum$=0.9$ and weight decay $=10^{-4}$. The comparison results with BWN and TWN are listed in Tab. \ref{tab:imagenet}.

\begin{table}[ht]
\centering
\begin{tabular}{|c|c|c|c|c|}
  \hline			
 Network & Bit-width & Method & Top-1 & Top-5 \\
  \hline
  \multirow{5}{*}{ResNet-18} 
  & 32 (float)
  	&  & 69.6 & 89.0 \\  
  	\cline{2-5}
  & \multirow{2}{*}{1 (binary)} 
  	& BWN  & 60.8 & 83.0 \\ 
    	\cline{3-5} &	
    & Ours & {\bf 63.2} & {\bf 85.1} \\  
  	\cline{2-5}
  & \multirow{2}{*}{2 (ternary)} 
  	& TWN  & 61.8 & 84.2 \\ 
    	\cline{3-5} &
    & Ours & {\bf 66.5} & {\bf 87.3} \\  
    \cline{2-5}
  \hline 
\end{tabular}
\medskip
\caption{ImageNet validation accuracies.}\label{tab:imagenet}
\end{table}

\bigskip

\section{Convergence Analysis}\label{sec:converg}
In this section, we analyze the convergence property of the proposed BinaryRelax.
More precisely, we will focus on Phase II of BinaryRelax (i.e., BinaryConnect):
\begin{equation}\label{eq:bc}
\begin{cases}
y^{k+1} = y^k - \gk \nabla f_k(x^{k}) \\
x^{k+1} = \proj(y^{k+1}).
\end{cases}
\end{equation}
Although the convergence of BinaryConnect at a small learning rate is observed empirically (as seen from our experiments), the only convergence results, to our knowledge, were proved in \cite{Goldstein_17}, in terms of the objective value of float ergodic averages  $\Big\{f\Big(\frac{\sum_{i=1}^k y^i}{k}\Big)\Big\}$ under convexity assumption. Moreover, the quantization set $\Q$ considered in \cite{Goldstein_17} is quite different. They constrain the quantized weights to $\{0,\pm\Delta, \pm2\Delta,\dots\}$, where $\Delta>0$ is some fixed resolution. In this case, $\Q$ is an unbounded $\Delta$-lattice in $\mathbb{R}^n$. 
Recall from section \ref{subsec:quant}
that our $\Q$ takes the form of $\cup_{i=1}^p \Li_i$ with each $\Li_i$ being a line passing through the origin. This assumption on $\Q$ generalizes the binary and ternary cases in the existing literature such as \cite{xnornet_16,twn_16}. Without assuming the convexity of $f$, we will show the sequence $\{x^k\}$ generated by the iteration (\ref{eq:bc}) subsequentially converges in expectation to an approximate critical point. To establish the convergence, we need to exploit the property of the set $\Q$ being the union of line subspaces by introducing several technical lemmata in section \ref{subsec:prelim}. The analysis here cannot be readily extended to the setup of \cite{Goldstein_17} or other problems under general discrete constraint. 

\subsection{Preliminaries}\label{subsec:prelim} 
We have the following basic assumptions.
\begin{assump}
$f(x)$ is bounded from below. Without loss of generality, we assume the lower bound is 0.
\end{assump}
\begin{assump}
$f(x)$ is $L$-Lipschitz differentiable, i.e., for any $x, y \in\R^n$, we have
$$
\|\nabla f(x) - \nabla f(y)\| \leq L \|x-y\|.
$$
\end{assump}
\begin{assump}
$\E[\|\nabla f(x^k) - \nabla f_k(x^k) \|^2] \leq \sigma^2$ for all $k\in\mathbb{N}$, where the expectation is taken over the stochasticity of the algorithm (i.e., random selection of $f_k$).
\end{assump}

Our proof relies on the following technical lemmata that exploit the structure of set $\Q$.

\begin{lemma}[Approximate orthogonality]\label{lem:othor}
Let $\{y^k\}, \{x^k\}$ be defined in (\ref{eq:bc}). There exists $\alpha_k\geq 0$, such that 
$$
\alpha_k\|x^{k+1}-x^k\|^2 + \|y^k-x^k\|^2 = \|y^k-x^{k+1}\|^2.
$$
\end{lemma}

\begin{proposition}\label{prop:alpha}
Let $\theta_{\min}$ be the smallest angle formed by any two line subspaces in $\Q$. If $\|x^{k+1} -x^k\|<\|x^k\| \sin \theta_{\min}$, then $\alpha_k = 1$ in Lemma \ref{lem:othor}. Moreover, $\alpha_k$ may have to be $0$ only when $\|y^k-x^k\| = \|y^k-x^{k+1}\|$ and $\nabla f_k(x^k) \perp \Li_i$ with $\Li_i$ containing $x^{k+1}$.
\end{proposition}
The above proposition implies that $\alpha_k$ is generally positive and approaches 1 when the relative change in consecutive iterates is getting small. 

\begin{lemma}[Alternative update]\label{lem:alt_update}
Let $\{x^k\}$ be defined in (\ref{eq:bc}). Suppose $x^{k+1}\in\Li_i\subset\Q$ with $\Li_i$ being some line subspace and define $\txk:=\mathrm{proj}_{\Li_i}(y^k)$, then 
$$
x^{k+1}= \arg\min_{x\in\Li_i} \|x-(\txk-\gk \nabla f_k(x^k) )\|^2.
$$
Moreover, $x^{k+1}$ is a local minimizer of the following problem
\begin{equation}\label{eq:optim}
\min_{x\in\Q} \|x-(\txk-\gk \nabla f_k(x^k) )\|^2.
\end{equation}
\end{lemma}

\begin{lemma}\label{lem:upper}
Let $\ak$ and $\tilde{x}^k$ be defined in Lemma \ref{lem:othor} and \ref{lem:alt_update}, resp., it holds that
$$
\|x^{k+1} - \tilde{x}^k\|^2\leq \alpha_k\|x^{k+1} - x^k\|^2.
$$
\end{lemma}

The following descent lemma will be useful.
\begin{lemma}[Descent lemma \cite{bertsekas1999nonlinear}]\label{lem:descent}
For any $x,y$, it holds that
$$
f(x)\leq f(y) + \la \nabla f(y), x-y \ra + \frac{L}{2}\|x-y\|^2.
$$
\end{lemma}

We recall the definition of subdifferential for proper and lower semicontinuous functions.
\begin{definition}[Subdifferential \cite{mordukhovich_06, Rockafellar_09}]
Let  $h: \mathbb{R}^n \rightarrow (-\infty, +\infty]$ be a proper and lower semicontinuous function. We define $\mathrm{dom} (h):=\{x\in \mathbb{R}^n: h(x)<+\infty\}.$ For a given $x\in \mathrm{dom}(h)$, the Fr$\acute{e}$chet subdifferential of $h$ at $x$, written as $\hat{\partial}h (x)$, is the set of all vectors $u\in \mathbb{R}^n$ which satisfy
  $$\lim_{y\neq x}\inf_{y\rightarrow x}\frac{h(y)-h(x)-\langle u, y-x\rangle}{\|y-x\|}\geq 0.$$
When $x\notin \mathrm{dom} (h)$, we set $\hat{\partial}h(x)=\emptyset$.
The (limiting) subdifferential, or simply the subdifferential, of $h$ at $x\in \mathbb{R}^n$, written as $\partial h(x)$, is defined through the following closure process
$$\partial h(x):=\{u\in\mathbb{R}^n: \exists \, x^k\rightarrow x, \; h(x^k)\rightarrow h(x)~\textrm{and}~ u^k\in \hat{\partial}h(x^k)\rightarrow u~\textrm{as}~k\rightarrow \infty\}.$$

\end{definition}

\subsection{Main results} We are in the position to present the convergence results, which are established under an approximate orthogonality condition on $\alpha_k$ in Lemma \ref{lem:othor}.
\begin{theorem}\label{thm}
Let $\{x^k\}$ be the sequence generated by (\ref{eq:bc}). Suppose there exist $ \ubar{\alpha}, \bar{\alpha}, \gamma>0$ such that $\uba\leq\alpha_k \leq \ba$ and $\gamma_{k+1}\leq\gk\leq \gamma<\frac{\uba}{2L}$ for all $k\in\mathbb{N}$. Then
$$
\lim_{k\to\infty}\E[\|x^{k+1}-x^k\|^2] = 0,
$$ 
if $\sum_{k=0}^{\infty} \gk^2<\infty$. If further $\sum_{k=0}^\infty\gk = \infty$, we have

$$\liminf_{k\to\infty} \; \E[\dis(\0,\partial h(x^k))^2] \leq 3\sigma^2\left(\frac{4\ba}{\uba^2} + 1\right),$$ 
where $h = f + \chi_\Q$ is the overall objective function.
\end{theorem}

\bigskip

\section{Concluding Remarks}\label{sec:rem}
From optimization point of view, we proposed BinaryRelax, a novel relaxation approach for training quantized neural networks. Our algorithm iterates between a hybrid gradient step for updating the float weights and a weighted average of the computed float weights and their quantizations. We increase slowly the parameter that controls the average to drive the weights to the quantized state. In order to get the purely quantized weights, exact quantization replaces the weighted average in the second phase of training. Extensive experiments shows that with about the same training cost, BinaryRelax is consistently better than its BinaryConnect counterpart in terms of validation accuracy. It has clearer advantage on larger networks, which yield more complex landscape of the training loss with spurious local minima. In addition, BinaryRelax is provably convergent in expectation under an approximate orthogonality condition, which is another contribution of this paper.

\bibliographystyle{siamplain}
\bibliography{references}

\bigskip

\section*{Appendix: Technical proofs}

\begin{proof}[Proof of Proposition \ref{prop:global}]
Since $x^*_\lambda$ is the global minimizer of (\ref{eq:binaryrelax}),
$$
f_\Q^*\geq f(x^*_\lambda) + \frac{\lambda}{2}\dis(x^*_\lambda,\Q)^2\geq f^* + \frac{\lambda}{2}\dis(x^*_\lambda,\Q)^2,
$$
where $f^* = \min_{x\in\R^n} f(x)>-\infty$. So
$$
\dis(x^*_\lambda, \Q) \leq \sqrt{\frac{2(f^*_\Q-f^*)}{\lambda}}\to 0, \mbox{ as } \lambda\to\infty.
$$
Denote $x^*_{\lambda,\Q}= \proj(x^*_\lambda)$, then $\|x^*_{\lambda,\Q}-x^*_\lambda\|\to0$ as $\lambda\to\infty$.
Since $f^*_\Q$ is the minimum in $\Q$, further we have 
$$
f(x^*_\lambda) + \frac{\lambda}{2}\dis(x^*_\lambda,\Q)^2 \leq f^*_\Q\leq f(x^*_{\lambda,\Q})\to f(x^*_\lambda), \mbox{ as } \lambda\to\infty.
$$
Therefore, $\lim_{\lambda\to\infty} f(x^*_\lambda)=f^*_\Q$.
\end{proof}

\bigskip

\begin{proof}[Proof of Proposition \ref{prop:approx_quant}]
Problem (\ref{eq:relax_proj}) amounts to
$$
\min_{x}\min_{z\in\Q} \;  \frac{1}{2}\|x-y^{k}\|^2 + \frac{\lambda}{2}\|z-x\|^2 = 
\min_{z\in\Q}\min_{x} \;  \frac{1}{2}\|x-y^{k}\|^2 + \frac{\lambda}{2}\|z-x\|^2. 
$$
With fixed $z\in\Q$, the inner problem is minimized at 
$x = \frac{\lambda \, z + y^{k}}{\lambda +1}$. Then it reduces to
\begin{align*}
 z^* = & \arg\min_{z\in\Q} \frac{1}{2} \left\|\frac{\lambda z + y^{k}}{\lambda +1}-y^{k}\right\|^2
+ \frac{\lambda}{2}\left\|z-\frac{\lambda z + y^{k}}{\lambda +1}\right\|^2 \\
= & \arg\min_{z\in\Q}\|z - y^{k}\|^2 = \proj(y^{k}).
\end{align*}
Therefore, $x^{k} = \frac{\lambda \, \proj(y^{k}) + y^{k}}{\lambda+1}$ is the optimal solution.
\end{proof}

\bigskip

\begin{proof}[Proof of Proposition \ref{thm:local}]
Proof by contradiction. Let us assume $x^*\in\Q$ is a local minimizer of problem (\ref{eq:binaryrelax}) and $\nabla f(x^*)\neq \0$. Then for any point $x$ in the neighborhood of $x^*$, we have
$$
f(x^*)\leq f(x) + \frac{\lambda}{2} \dis(x,\Q)^2\leq f(x) + \frac{\lambda}{2}\|x-x^*\|^2.
$$
Set $x = x^* - \beta \nabla f(x^*)$ with a small $\beta>0$. The above inequality reduces to
\begin{equation}\label{eq:local}
f(x^*)\leq f(x^* - \beta \nabla f(x^*)) + \frac{\lambda\beta^2}{2}\|\nabla f(x^*)\|^2.
\end{equation}
On the other hand, by Taylor's expansion,
\begin{equation}\label{eq:taylor}
f(x^* - \beta \nabla f(x^*)) = f(x^*) -\beta\|\nabla f(x^*)\|^2 + o(\beta).
\end{equation}
Combining (\ref{eq:local}) and (\ref{eq:taylor}), we have
$$
\|\nabla f(x^*)\|^2 \leq \frac{\lambda\beta}{2}\|\nabla f(x^*)\|^2 + o(1),
$$
which leads to a contradiction as we let $\beta\to0$.
\end{proof}

\bigskip

\begin{proof}[Proof of Lemma \ref{lem:othor}]
Since $x^k, x^{k+1}\in\Q$ and $x^k = \proj(y^k)$, it holds that $\|y^k-x^k\|^2 \leq \|y^k-x^{k+1}\|^2$, i.e., $\alpha_k\geq 0$. 
\end{proof}

\bigskip

\begin{proof}[Proof of Proposition \ref{prop:alpha}]
Since the only intersection of the line subspaces is the origin, the distance between $x^k$ and any other line is at least $\|x^k\|\sin\theta_{\min}$. If $\|x^{k+1}-x^k\|< \|x^k\| \sin\theta_{\min}$, then $x^k$ and $x^{k+1}$ must lie in the same line, and therefore $\alpha_k = 1$. On the other hand, if $\alpha_k$ can only be $0$, then it must hold that $\|y^k - x^k\| = \|y^k - x^{k+1}\|$ and $x^k \neq x^{k+1}$, meaning that $x^{k+1}$ is a different projection of $y^k$ onto $\Q$. Moreover, since the projection of $y^{k+1} = y^k -\gamma_k \nabla f_k(x^k)$ onto $\Q$ is also $x^{k+1}$. Suppose $x^{k+1}\in\Li_i\subset\Q$, then $\nabla f_k(x^k)\perp \Li_i$.
\end{proof}

\bigskip

\begin{proof}[Proof of Lemma \ref{lem:alt_update}] By the assumption, we have
$$
x^{k+1} = \mathrm{proj}_{\Li_i}(y^{k}-\gk \nabla f_k(x^k)) = \mathrm{proj}_{\Li_i}(\txk-\gk \nabla f_k(x^k) + y^k- \txk).
$$
Note that $y^k - \txk \perp \Li_i$ (see Fig.~\ref{fig:proj}), then 
$$
x^{k+1} = \mathrm{proj}_{\Li_i}(\txk -\gk \nabla f_k(x^k)).
$$
So $x^{k+1}$ is the closest point to $\txk -\gk \nabla f_k(x^k) $ on $\Li_i$. If $\txk -\gk \nabla f_k(x^k) =\0$, then $x^{k+1} = \0$ is the global minimizer of (\ref{eq:optim}). Otherwise, $x^{k+1}\neq\0$. Since the line subspaces that constitute $\Q$ only intersect at the origin, there exists a neighborhood $\mathcal{N}$ of $x^{k+1}$ such that $\mathcal{N}\cap\Q\subset \Li_i$. Therefore, $x^{k+1}$ is a local minimizer of problem (\ref{eq:optim}).
\end{proof}

\begin{figure}
\begin{center}
\includegraphics[width=.53\textwidth,height = 0.4\textwidth]{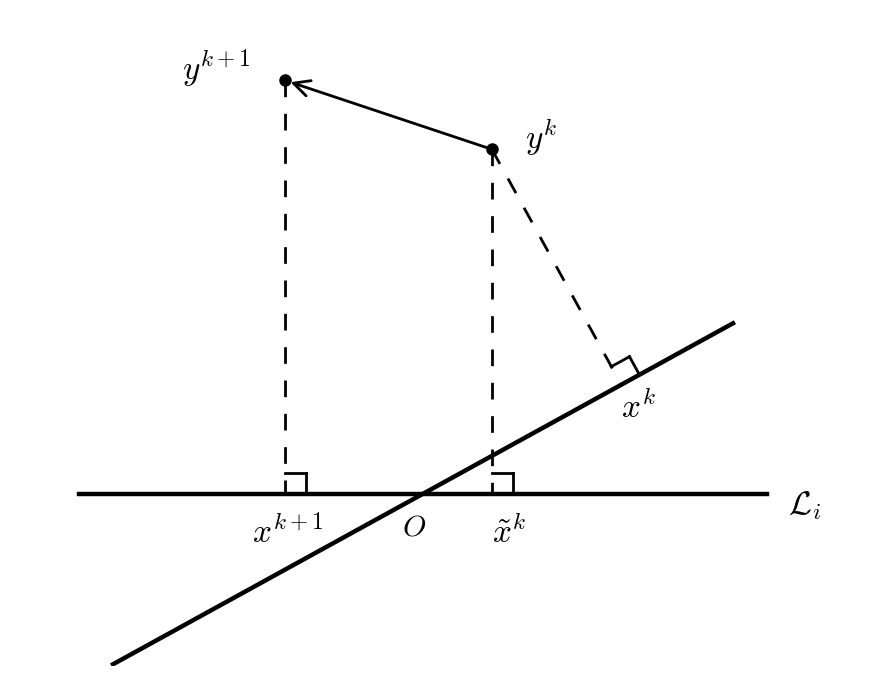}
\end{center}
\caption{Illustration of Lemma \ref{lem:alt_update}. $y^{k+1} = y^{k} -\gamma_k \nabla f_k(x^{k})$, so $x^{k+1}$ is also the projection of $\txk - \gamma_k \nabla f_k(x^{k})$ onto $\Li_i$. }\label{fig:proj}
\end{figure}

\bigskip

\begin{proof}[Proof of Lemma \ref{lem:upper}]
 Using the facts $x^k = \proj(y^k)$, $\txk = \mathrm{proj}_{\Li_i}(y^k)\in\Q$ and $x^{k+1}\in\Li_i$ and invoking Lemma \ref{lem:othor}, we have
\begin{align*}
& \|x^{k+1}-\txk\|^2 = \|y^k-x^{k+1}\|^2 - \|y^k-\txk\|^2 \\
\leq & \|y^k-x^{k+1}\|^2 - \|y^k-x^k\|^2 = \ak \|x^{k+1}-x^k\|^2. 
\end{align*}
\end{proof}

\bigskip

\begin{proof}[Proof of Theorem \ref{thm}]
By Lemma \ref{lem:descent},
\begin{align}\label{eq:descent}
& f(x^{k+1})\leq f(x^k) + \la \nabla f(x^k), x^{k+1}-x^k \ra + \frac{L}{2}\|x^{k+1}-x^k\|^2 \notag \\
= & f(x^k) + \la \nabla f_k(x^k),x^{k+1}-x^{k}\ra + \la \nabla f(x^k) -\nabla f_k(x^k),x^{k+1}-x^{k}\ra + \frac{L}{2}\|x^{k+1} - x^k\|^2. 
\end{align}

The cross terms need care. We rewrite the update 
$x^{k+1} = \proj(y^k - \gamma_k\nabla f_k(x^k))$
as
\begin{equation*}
x^{k+1} = \arg\min_{x\in\Q} \; \la \nabla f_k(x^k), x\ra + \frac{1}{2\gamma_k}\|x- y^k\|^2.
\end{equation*}
Since $x^k\in \Q$, we have
$$\la \nabla f_k(x^k), x^{k+1}\ra + \frac{1}{2\gamma_k}\|x^{k+1} - y^k \|^2 
\leq \la \nabla f_k(x^k), x^k\ra + \frac{1}{2\gamma_k}\|x^k - y^k \|^2.
$$
Then by Lemma \ref{lem:othor},
\begin{equation}\label{eq:cross1}
 \la \nabla f_k(x^k), x^{k+1} - x^k \ra \leq  \frac{1}{2\gamma_k}(\|x^k - y^k \|^2 - \|x^{k+1} - y^k \|^2 ) \leq -\frac{\uba}{2\gamma_k}\|x^{k+1}-x^k\|^2.
\end{equation}
By Young's inequality,
\begin{equation}\label{eq:cross2}
\la \nabla f(x^k) -\nabla f_k(x^k),x^{k+1}-x^{k}\ra  \leq \frac{\gamma_k}{\uba}\|\nabla f(x^k) -\nabla f_k(x^k)\|^2 +  \frac{\uba}{4\gamma_k}\|x^{k+1}-x^k\|^2.
\end{equation}
Combining (\ref{eq:descent}), (\ref{eq:cross1}) and (\ref{eq:cross2}) and taking the expectation gives
\begin{equation}\label{eq:Lyap}
\E[f(x^{k+1})]\leq \E[f(x^k)]-\frac{\uba -2\gk L}{4\gk}\E[\|x^{k+1} -x^k\|^2] + \frac{\gk\sigma^2}{\uba}.
\end{equation}
Multiplying (\ref{eq:Lyap}) by $\gk$ and using $\alpha_k\geq \uba>0$, $\gamma_{k+1}\leq \gk\leq \gamma<\frac{\uba}{2L}$ and $f\geq0$, we obtain
$$
\gamma_{k+1}E[f(x^{k+1})]\leq \gamma_{k}\E[f(x^{k+1})]\leq \gamma_{k}\E[f(x^{k})] - (\uba - 2\gamma L)\E[\|x^{k+1}-x^k\|^2] + \frac{\gk^2\sigma^2}{\uba}
$$
Rearranging terms in the above inequality and taking the sum over $k$, we have
$$
(\uba - 2\gamma L)\sum_{k=0}^\infty \E[\|x^{k+1}-x^k\|^2]\leq \gamma f(x^0)-\lim_{k\to\infty}\gamma_{k}\E[f(x^{k})]+\frac{\sigma^2}{\uba}\sum_{k=0}^\infty\gk^2<\infty.
$$
Therefore, $\lim_{k\to\infty}\E[\|x^{k+1}-x^k\|^2]= 0$.

\medskip

Next we prove the second claim. By Lemma \ref{lem:alt_update}, the first-order optimality condition of (\ref{eq:optim}) holds at $x^{k+1}$. So
$$
\0\in \nabla f_k(x^k) + \frac{x^{k+1}-\txk}{\gk} + \partial \chi_\Q(x^{k+1}),
$$
which implies 
$$-\frac{x^{k+1}-\txk}{\gk} - \nabla f_k(x^k) + \nabla f(x^{k+1})\in \nabla f(x^{k+1}) + \partial \chi_\Q(x^{k+1}) = \partial h(x^{k+1}).
$$
Therefore, 
\begin{align}\label{eq:est}
& \E[\dis(\0, \partial h(x^{k+1}))^2] \notag \\
\leq & \E\left[\left\|-\frac{x^{k+1}-\txk}{\gk} - \nabla f_k(x^k) + \nabla f(x^{k+1})\right\|^2\right]  \notag \\
\leq & 3\left(  \E \left[\frac{\|x^{k+1}-\txk\|^2}{\gk^2} \right] + \E[\|\nabla f_k(x^k)-\nabla f(x^k) \|^2] + \E[\|\nabla f(x^k) - \nabla f(x^{k+1}) \|^2] \right) \notag \\
\leq & 3\left( \ba\E\left[\frac{\|x^{k+1}-x^k\|^2}{\gk^2} \right] + \sigma^2 + L^2 \E[\|x^{k+1} - x^{k} \|^2]  \right).
\end{align}
The second inequality above holds because of Cauchy-Schwarz inequality. In the last inequality, we used Lemma \ref{lem:upper} and the assumption that $f$ is $L$-Lipschitz differentiable. We want to bound $\liminf_{k\to\infty} \E \left[ \frac{\|x^{k+1}-x^k\|^2}{\gk^2} \right]$. From (\ref{eq:Lyap}) it follows that 
$$
\gk\left( (\uba-2\gk L) \E\left[ \frac{\|x^{k+1}-x^k\|^2 }{4\gk^2}\right]- \frac{\sigma^2}{\uba}  \right) \leq \E[f(x^k) - f(x^{k+1})].
$$
Summing the above inequality over $k$ yields
\begin{align*}
\sum_{k=0}^{\infty} \gk\left( (\uba-2\gk L) \E\left[ \frac{\|x^{k+1}-x^k\|^2}{4\gk^2}\right]- \frac{\sigma^2}{\uba} \right)
\leq f(x^0) < \infty.
\end{align*}
Since $\gk>0$ and $\sum_{k=1}^\infty \gk=\infty$, we must have 
$$
\liminf_{k\to\infty} \; (\uba-2\gk L)\E\left[ \frac{\|x^{k+1}-x^k\|^2 }{4\gk^2}\right]- \frac{\sigma^2}{\uba}   \leq 0,
$$
and thus
$$
\liminf_{k\to\infty} \; \E\left[\frac{\|x^{k+1}-x^k\|^2}{\gk^2}\right]
\leq \lim_{k\to\infty}\frac{4\sigma^2}{\uba(\uba -2\gamma_k L)} = \frac{4\sigma^2}{\uba^2}.
$$
Finally, from (\ref{eq:est}) it follows that
\begin{equation*}
\liminf_{k\to\infty} \; \E[\dis(\0, \partial h(x^{k}))^2] \leq 3\sigma^2\left(\frac{4\ba}{\uba^2} + 1\right),
\end{equation*}
which completes the proof.
\end{proof}

\end{document}